\documentclass[conference]{IEEEtran}
\IEEEoverridecommandlockouts
\usepackage{cite}
\usepackage{amsmath,amssymb,amsfonts}
\usepackage{algorithmic}
\usepackage{graphicx}
\usepackage{textcomp}
\usepackage{xcolor}
\usepackage{mathtools}
\usepackage{amsthm}
\usepackage{ulem}
\usepackage{algorithm}
\newtheorem{Theorem}{Theorem}[section]

\newtheorem{lemma}{Lemma}[section]

\newtheorem{Assumption}{Assumption}[section]

\newtheorem{Definition}{Definition}[section]
\theoremstyle{remark}
\newtheorem{Remark}{Remark}[section]
\newcommand{\e}{\mathbf{e}}

\newcommand{\x}{\mathbf{x}}
\newcommand{\y}{\mathbf{y}}
\newcommand{\z}{\mathbf{z}}

\newcommand{\bv}{\mathbf{v}}
\newcommand{\reals}{\mathbb{R}}

\newcommand{\hyf}{\hat{\y}_f}

\newcommand{\bP}{\mathbf{P}}
\newcommand{\bE}{\mathbf{E}}
\newcommand{\F}{\mathbf{F}}

\newtheorem{notation}{Notation}

\graphicspath{figures}

\def\BibTeX{{\rm B\kern-.05em{\sc i\kern-.025em b}\kern-.08em
    T\kern-.1667em\lower.7ex\hbox{E}\kern-.125emX}}
\begin{document}

\title{PAC-Bayesian theory for stochastic LTI systems}
\author{Deividas Eringis, John Leth and Zheng-Hua Tan, Rafal Wisniewski\\ Alireza Fakhrizadeh Esfahani, Mihaly Petreczky
\thanks{Deividas Eringis, Johns Leth,  Zheng-Hua Tan and Rafal Wisnieswski are with Dept. of Electronic Systems, Aalborg University, Denmark, \{der,jjl,zt,raf\}@es.aau.dk.}
\thanks{Mihaly Petreczky and Alireza Fakhrizadeh Esfahani is with Laboratoire Signal et Automatique de Lille (CRIStAL) Lille, France, mihaly.petreczky@centralelille.fr, alireza.fakhrizadeh.esfahani@gmail.com}
}


\maketitle


\begin{abstract}
In this paper we derive a PAC-Bayesian error bound for autonomous stochastic LTI state-space models. The motivation for deriving such error bounds is that they will allow deriving similar error bounds for more general dynamical systems, including recurrent neural networks. In turn, PAC-Bayesian error bounds are known to be useful for analyzing machine learning algorithms and for deriving new ones. 
\end{abstract}
\section{Introduction}
The goal of this paper to present a PAC-Bayesian error bound for learning autonomous stochastic linear time-invariant (LTI)
state-space representations (\emph{LTI systems} for short). Autonomous stochastic LTI systems are widely used to model time series, they correspond to ARMA models. They represent one of the simplest classes of dynamical systems, and their learning theory has a rich history. In particular, the learning of
stochastic LTI systems has been the subject of system identification \cite{LjungBook}.  Despite the large body of literature for learning LTI models, there has been no result on PAC-Bayesian error bounds for LTI models.

In order to explain what PAC-Bayesian framework is, it is useful to recall from \cite{LjungBook} the learning problem for LTI systems.
Recall from \cite{LjungBook} that stochastic LTI models can be viewed bot as generators of an output process (its distribution), and 
as predictors which predict the current value of an output process based on its past values.
These two views are equivalent. 
In this paper we will view dynamical systems as predictors, which use past outputs to predict future outputs. 
Informally, following  \cite{LjungBook} the learning problem for autonomus stochastic LTI models can be formulated as follows. Consider a process $\y$ which is generated by a black-box system. The goal is to find an LTI model which  uses the past of the observed process $\y$ to predict the current value of $\y$, and such that the prediction error is the smallest possible. The prediction error is measured by using a so called loss function. 
There are two types of prediction error to consider:
the \emph{generalization error} and the \emph{empirical loss}
The generalization error is mathematical expectation of the difference between the current values of $\y$ and the one predicted by the model, if the model can use all the past observations of $\y$ to predict future ones.
The generalization error is a number which tells us how well the model performs in average. In contrast, the average empirical error is a random variable, each realization of which corresponds to the actual prediction error for a finite number of observations sampled from $\y$.

As a rule, algorithms for learning LTI models \cite{LjungBook} chose the model parameters in such a manner that the empirical error is small, when evaluated for the data available for learning. 
However, the fact that a certain model renders the empirical error small does not necessarily mean that the generalization error will be small too. In fact, one of  the main theoretical challenge is to prove this implication for certain classes of
learning problems and algorithms.

Informally, the immediate goal of PAC-Bayesian framework \cite{mcallester-99,mcallester-03b,guedj2019primer}
is to derive an upper 
bound on the average generalization error, which involves the 
average empirical error of models and a number of tuning parameters. 
The averaging is done on the space of all possible models, using a 
probability distribution on the space of models.
On of the tuning parameters is the prior distribution on the space of inputs, which represents our hypothesis on the model structure; 
We will refer to this
inequality as the \emph{PAC-Bayesian inequality}. 
Once a PAC-Bayesian inequality has be derived, learning algorithms can be formulate as follows.
\begin{enumerate}
\item Choose the value of the prior distribution and other tuning parameters.
\item Find probability
      distribution which minimizes the PAC-Bayesian upper bound for the
      average generalization error.  Since this upper bound involves
      the empirical loss, we will tend to assign higher
      probability to models which fit well the training data. However, this
      will be tempered by the need to take into account the prior distribution
      and the tuning parameters. 
      This step can be performed using sampled data.
      The thus obtained probability distribution is analogous to 
      the posterior distribution in the Bayesian setting.
\item The final model can be obtained by either sampling randomly from the
      posterior distribution, choosing the element with the maximal density (if the posterior distribution is absolutely continuous), or by taking the mean of the posterior distribution. 
\end{enumerate}
Various choices of the prior distribution on the space of model and of the tuning parameters which enter the PAC-Bayesian inequality will lead to different
learning algorithms. In fact, it can be
shown \cite{guedj2019primer,nips-16} that known learning algorithms can be derived from PAC-Bayesian inequality. The thus derived algorithms will automatically be equipped with
an upper bound on the generalization error. In turn, we can use this upper
bound for a systematic analysis of corresponding learning algorithm.

The PAC-Bayesian learning theory \cite{mcallester-99,mcallester-03b}) is one of
the several approaches for studying theoretical properties of machine learning algorithms. This approach is subject of intensive research, see \cite{guedj2019primer} for a recent survey. It has demonstrated its ability to provide computable generalization bounds on popular machine learning algorithms, such as neural networks \cite{Dziugaite2017} and SVMs \cite{ambroladze-06}. Moreover, as its name suggests, PAC-Bayesian framework combines   the classical Probably Approximately Correct theory based on a frequentist view with Bayesian inference, see \cite{zhang-06,grunwald-2012,alquier-15,nips-16,ShethK17} for a more detailed discussion on the topic.

The motivation for studying PAC-Bayesian error bounds for LTI models is as follows.  Autonomous stochastic LTI models are a special case of stochastic LTI models with inputs. In fact, it is known \cite{LjungBook,LindquistBook,Katayama:05} that learning stochastic LTI models with inputs  can be decomposed into learning autonomous stochastic LTI models and learning deterministic LTI models with inputs. Hence, in order to derive PAC-Bayesian error bounds for stochastic LTI models with inputs the first step is to derive such error bounds for autonomous stochastic LTI models. In turn, stochastic LTI models with inputs are special cases of recurrent neural networks (RNN) and their various modifications, such as LSTM networks. Hence,
deriving PAC-Bayesian error bounds for autonomous stochastic LTI models is the first step towards deriving PAC-Bayesian 
error bounds for stochastic LTI models with inputs, and then for RNNs and LSTMs. In turn, this will allow to understand the theoretical limitations of learning algorithms for RNNs. 

Despite an impressive body of literature, PAC-Bayesian error bounds are not available for learning LTI systems. 
One of the reasons for this is that most of the existing literature dealt with bounded and Lipschitz loss functions
\cite{shalev2014understanding,Dziugaite2017,alquier2013prediction}, which is typical for classification problem, but not for regression problems. In \cite{nips-16} PAC-Bayesian bounds for linear regression with a quadratic loss function was developed, later this bound was improved and extended to non i.i.d. data in 
\cite{shalaeva2019improved}.  In particular, in \cite{shalaeva2019improved} the derived PAC-Bayesian error bound for non i.i.d linear regression problem was applied to learning ARX models.  Note that in \cite{shalaeva2019improved} control inputs were allowed, while in the current paper we consider only autonomous systems. 
Since ARX models without inputs correspond to AR models, and 
the latter class represents a special case of autonomous stochastic LTI  state-space representations, the present paper can be viewed as an extension of \cite{shalaeva2019improved} to stochastic LTI state-space representations. 



The structure of the paper is as follows. 
In Section \ref{sect:learning:gen} we present the formal framework for PAC-Bayesian learning of linear dynamical systems. In Section \ref{lti:bayesian} we present the main result of the paper which is an analytic PAC-Bayesian error bound for stochastic LTI models. In Section \ref{sect:num} we present a numerical example for illustrating the theoretical results. 
In Section \ref{sect:concl} we present the conclusions and directions for future research.

\section{PAC-bayesian learning for linear dynamical systems}
\label{sect:pac:learning:gen}
The purpose of this section is to formulate the learning problem for stochastic dynamical systems and to present a brief overview of PAC-Bayesian framework.

\subsection{The learning problem}
\label{sect:learning:gen}

In order to formalize the problem, let
us denote by $\mathcal{Y}=\mathbb{R}^{p}$ the set of output values. Moreover, we fix a probability space $(\Omega,\bP,\F)$, where $\F$ is a $\sigma$-algebra on $\Omega$ and $\bP$ is a probability measure on $\F$, see for example \cite{KlenkeAchim2020PTAC} for the terminology. We use $\bE$ to denote the
 mathematical expectation with respect to the probability measure $\bP$.
Consider a stochastic process $\y$ on this probability space, taking values in $\mathcal{Y}$, and with time axis $\mathbb{Z}$. That is, for
any $t \in \mathbb{Z}$, $\y(t):\Omega \rightarrow \mathcal{Y};~\omega\mapsto \y(t)(\omega)=y$ is a random
variable on $(\Omega,\bP,\F)$. For technical reasons,
we assume that $\y$ is 
\begin{itemize}
    \item Stationary: its finite dimensional distributions are invariant under time displacement
    \begin{multline*}
        \bP(\y(t_1+t)\leq y_1,\dots,\y(t_k+t)\leq y_k)\\
        =\bP(\y(t_1)\leq y_1,\dots,\y(t_k)\leq y_k)
    \end{multline*}
    for any $t,t_1,\dots,t_k$ and $k$.
    \item Square-integrable: $\bE[||\y(t)||^2_2]<\infty$ for any $t$, with $||\cdot||_2$ denoting the Euclidean 2-norm.
    \item Zero mean: $\bE[\y(t)]=0$ for any $t$. 
\end{itemize}
Recall that the first two assumptions imply constant expectation and that the covariance matrix $Cov(\y(t),\y(s))=\bE[(\y(t)-\bE[\y(t)])(\y(s)-\bE[\y(s)])^T]$ only depend on $t-s$.  

The goal of learning is to estimate \emph{models} for predicting $\y$
base on a training set $\{\y(t)(\omega)\}_{t=0}^{N}$ formed by a finite portion of a sample path of $\y$.

 Let $\mathcal{Y}^{*}=\bigcup_{k=1}^{\infty} \mathcal{Y}^k$ where here $\bigcup$ denote disjoint union. By slight abuse of notation we simply write $y=(y_1,\ldots,y_k)$ in place of the more correct $(y,k)=((y_1,\ldots,y_k),k)$ for an element in $\mathcal{Y}^{*}$. Moreover, whenever necessary we associate with $\mathcal{Y}^k$ the topology generated by $||\cdot||_2$, and the Borel $\sigma$-algebra $\mathcal{B}_k$ generated by the open sets of $\mathcal{Y}^k$.
 
In our setting, models will be a subset of measurable functions of the form
\begin{equation}
\label{model:eq1}
 f: \mathcal{Y}^{*} \rightarrow \mathcal{Y}
\end{equation}
In particular, $f^{-1}(B)\in \mathcal{B}_k$ for all $k$ whenever $B\in \mathcal{B}_1$. For the easy of notation, for any $(y_1,\ldots,y_k) \in \mathcal{Y}^k$,
we denote $f((y_1,\ldots,y_k))$ by $f(y_1,\ldots,y_k)$.

In this paper, we will restrict attention to models which arise from stable transfer functions.  More precisely, we will use the following definition.
\begin{Definition}[Model]
\label{model:def}
 A model is a function $f$ of the form \eqref{model:eq1} such that there 
 exists a sequence of $p \times p$ matrices $\{M_k\}_{k=1}^{\infty}$
 such that the sequence is absolutely summable, i.e.,
 $\sum_{k=0}^{\infty} \|M_k\|_2 < \infty$, and 
 \begin{equation}
 \label{model:def:eq1}
 f(y_1,\ldots,y_k)=\sum_{i=1}^{k} M_i y_i 
 \end{equation}
\end{Definition}
Intuitively, a model $f$ satisfying the conditions of 
Definition \ref{model:def} are input-output maps of linear systems defined by
a stable transfer function $G(z)=\sum_{k=1}^{\infty} M_k z^{-(k-1)}$. 

Intuitively, if $f$ is applied to some tuple of past value of $\y$, then the
resulting random process will be viewed as a prediction of the current value of $\y$ based on those past values, that is, 
\begin{notation}
In the sequel we use the following notation: $$\hyf(t\mid s) = f(\y(t-1),\ldots,\y(s)).$$
\end{notation}
The random variable 
$\hyf(t\mid s)$ is interpreted as the prediction of $\y(t)$
generated by $f$, based on the past values $\y(t-1),\ldots, \y(s)$, $s < t$.
Clearly,
\begin{equation}
\label{model:def:predict}
\hyf(t \mid s)=\sum_{k=1}^{t-s} M_k\y(t-k)
\end{equation}
\begin{Remark}
Notice that since $\y$ is stationary, it follows that the probability distribution of 
$\hyf(t \mid s)$  depends only on $t-s$
\end{Remark}

Naturally, we would like to find models which predict the output process well. The “quality” of the predictor $f$ is usually assessed through a measurable \emph{loss function} $\ell : \mathcal{Y}\times \mathcal{Y} \rightarrow [0,\infty)$.
The loss function evaluated at
$\ell(\hyf(t\mid s),\y(t))$ measures how well the predicted value
$\hyf(t\mid s)$ approximates the true value of $\y$
at $t$: the smaller $\ell(\hyf(t\mid s),\y(t))$ is, the smaller the prediction error is. Note that the use of the word smaller here needs to be quantified e.g., by $\bE$, as $\ell(\hyf(t\mid s),\y(t))$ is a random variable. Indeed, in the sequel we will refer to $\bE[l(\hyf(t\mid s),y(t))]$ as the prediction error   
\begin{Assumption}[Quadratic loss function]
In the sequel, we will assume that the loss function is quadratic, i.e.
$\ell(y,y^{'})=\|y-y^{'}\|_2^2=(y-y^{'})^T(y-y^{'})$.
\end{Assumption}
Some of the subsequent discussion, especially Theorem~\ref{thm:general-alquier} can be extended to other loss functions, but for the purposes of the paper
quadratic loss functions are sufficient. 

Then the quantity 
\begin{equation}
\label{model:def:finpred}
\bE[\ell(\hyf(t\mid s),\y(t))], 
\end{equation}
measures the mean difference between the actual process $\y(t)$ and 
the predicted value $\hyf(t \mid s)$ bases on $\{\y(\tau)\}_{\tau=s}^{t-1}$.
However, the expectation \eqref{model:def:finpred} is not the most
convenient  measure of the predictive power of a model, as it depends
on the prediction horizon $t-s$. In practice, the prediction
horizon tends to increase with the increase of the number of available
data points. For this reason, it is more convenient to consider the
prediction error as the beginning of the prediction horizon goes to infinity.

Intuitively, we expect that the quality of the prediction will increase with the growth of the horizon  $t-s$ used for prediction. 
In fact, we can state the following.
\begin{lemma}[Infinite horizon prediction, \cite{hannan}]\label{l:ihp}
The limit 
\begin{equation}
\label{model:inf:eq-1}
    \hyf(t)=\lim_{s \rightarrow -\infty} \hyf(t \mid s)
\end{equation}
exists
in the mean square sense for all $t$, the process $\hyf(t)$ is stationary,
\begin{equation}
     \label{model:inf:eq}
      \bE[\ell(\hyf(t),\y(t))]=\lim_{s \rightarrow -\infty} \bE[\ell(\hyf(t \mid s),\y(t))] 
\end{equation} 
and $\bE[\ell(\hyf(t),\y(t))]$ does not depend on $t$.
\end{lemma}

This prompts us to introduce the following definition.
\begin{Definition}[Generalization error of a model]
 The quantity
 \[ \mathcal{L}^\ell_{\y}(f)=\bE[\ell(\hyf(t),\y(t)]=\lim_{s \rightarrow -\infty} \bE[\ell(\hyf(t \mid s),\y(t))]
 \]
 is called the generalization error of the model $f$ when applied to
 process $\y$, or simply generalization error, when $\y$ is clear from
 the context. 
\end{Definition}
Intuitively, $\hyf(t)$ can be interpreted as the prediction of
$\y(t)$ generated by the model $f$ based on all past values of $\y$. As stated in Lemma~\ref{l:ihp} we consider the special case when 
$\hyf(t)$ is the mean-square limit of $\hyf(t \mid s)$ as
$s \rightarrow -\infty$. Clearly, for large enough $t-s$, 
the prediction error $\bE[\ell(\hyf(t \mid s),\y(t))]$ is close to
the generalization error. 

The goal of learning is to find a model from a set of possible models with the smallest possible prediction error, using a finite portion of the sample path of $\y$.  This can be achieved through minimizing a cost function which involves the so called \emph{empirical error}. 
Assume that we would like to learn a model from the time series
$S=\{\y(t)(\omega)\}_{t=0}^{N}$, for some $\omega \in \Omega$. 
The data $S$ represents a finite portion of sample path
$\{\y(t)(\omega)\}_{t \in \mathbb{Z}}$ of $\y$, and $N+1$ represents the number
of data points used for the learning problem. 
Let us define first the concept of empirical (error) loss.
\begin{Definition}
The empirical loss for a model $f$ and process $\y$ is defined by
\begin{equation}
    \hat{\mathcal{L}}^\ell_{\y,N}(f)=\frac{1}{N}\sum_{t=1}^{N} \ell(\hyf(t \mid 0), \y(t)).\label{eq:EmpLoss}
\end{equation}
\end{Definition}
Note that $\hat{\mathcal{L}}^{\ell}_{\y,N}(f):\Omega \rightarrow \mathbb{R}$ is a 
random variable, and for $\omega \in \Omega$, 
$\hat{\mathcal{L}}^{\ell}_{\y,N}(f)(\omega)$ corresponds to the average prediction
error produced by $f$ when applied to the samples
$\y(t)(\omega)$ for $t=0,\ldots,N$ successively, that is, 
if the training data is $S=\{\y(t)(\omega)\}_{t=0}^{N-1}$, then
\begin{multline*}
\hat{\mathcal{L}}^\ell_{\y,N}(f)(\omega)
=\frac{1}{N}\sum_{t=1}^{N} \ell(\hyf(t\mid 0)(\omega),\y(t)(\omega))\\
=\frac{1}{N}\sum_{t=1}^{N} \ell(f(\y(t-1)(\omega),\ldots,\y(0)(\omega)), \y(t)(\omega))\\
\end{multline*} 
i.e. $\hat{\mathcal{L}}^\ell_{\y}(f)(\omega)$ depends only on the training
data $S$. 

\begin{Remark}[Alternatives for definining the empirical loss and relationship with system identification]
\label{rem:emploss}
Note that the minimization of the empirical loss is quite a common method in system identification. However, for theoretical analysis usually not \eqref{eq:EmpLoss} is taken, but rather 
\begin{equation}
\label{eq:EmpLoss1}
    V_{\y,N}(f)=\frac{1}{N}\sum_{t=1}^{N} \ell(\hyf(t), \y(t)).
\end{equation}    
That is, instead of the prediction error using finite past 
$\{\y(s)\}_{s=0}^{N}$  they use the infinite past
$\{\y(s)\}_{=-\infty}^{N}$. As it was pointed out above,
$\lim_{s \rightarrow -\infty} (\hyf(t\mid s)-\hyf(t))=0$, where the limit is understood in the mean square sense. Using this observation and a Cesaro-mean like argument it can be shown that
\begin{lemma}
\label{lem:rem:emploss}
 $\lim_{N \rightarrow \infty} \hat{\mathcal{L}}_{\y,N}(f)-V_{\y,N}(f)=0$, where the limit is understood in the mean sense. 
\end{lemma}
In particular, as mean convergence implies convergence in probability, 
it is clear that for large enough $N$, $\hat{\mathcal{L}}_{\y,N}(f)$, $V_{\y,N}(f)$
will be close enough with a large probability. That is, for practical purposes minimizing
$V_{\y,N}(f)$ versus $\hat{\mathcal{L}}_{\y,N}(f)$ will not make a big difference. 
The reason for using $\hat{\mathcal{L}}_{\y,N}(f)$ instead of
$V_{\y,N}(f)$ is that we are interested in deriving PAC-Bayesian bounds, which are not asymptotic. This is in contrast to the usual approach of system identification, which focuses on asymptotic results, i.e., behavior of the algorithms as $N \rightarrow \infty$. 

\end{Remark}

Depending on various choices of the cost  function, which always involves
the empirical loss, different learning algorithms can be derived.
The challenge is to be able to characterize the generalization error
of the results of these learning algorithms, and to understand how the
choice of the learning algorithm influences this generalization error.

\subsection{PAC-Bayesian framework}
\label{sect:pac:gen}
To this end, in the machine learning community the
so called PAC-Bayesian framework was proposed. 
Next we  present the PAC-Bayesian framework in detail.
Let $\mathcal{F}$ be a set of models, somethimes referred to as the set of hypothesis. Assume that 
$\mathcal{F}$ can be parametrized by a parameter set $\Theta$, i.e., there
is a bijection $\Pi:\Theta \rightarrow \mathcal{F}$. For simplicity,
denote $\Pi(\theta)$ by $f_{\theta}$, $\theta \in \Theta$. That is, $f_{\theta }$ is the model which corresponds to the parameter $\theta$. Assume that $(\Theta,B_{\theta},m)$ is a measure space, where $B_{\Theta}$ is a $\sigma$-algebra on $\Theta$ and $m$ is a measure on $(\Theta,B_{\theta})$. Furthermore, assume that the function $\theta \mapsto \mathcal{L}^{\ell}_\y(f_{\theta})$ is measurable and 
integrable with respect to $(\Theta,B_{\theta},m)$. Assume that
the function $\Theta \times \Omega \ni (\theta, \omega) \mapsto \hat{\mathcal{L}}^{\ell}_{\y,N}(f_{\theta})$ is measurable and integrable with respect to the joint product measure $m \times \bP$.
Let $\rho$ be a probability density function on the measure space $(\Theta,B_{\theta},m)$, and $g:\mathcal{F} \rightarrow \mathbb{R}$ be a map such that
$\Theta \ni \theta \mapsto g(f_{\theta})$ is measurable and absolutely integrable. We then denote by
$E_{f \sim \rho} g(f)$ the integral
\begin{equation}
\label{model:average}
 E_{f \sim \rho} g(f)=\int_{\theta \in \Theta} \rho(\theta)g(f_{\theta})dm(\theta)
\end{equation}
Moreover, by a slight abuse of notation and terminology we sometimes write $\rho(f)$ for the value $\rho(\theta)$ whenever $f=f_{\theta}$ for $\theta \in \Theta$, and say that $\rho$ is a density on $\mathcal{F}$. With this notation and terminology in mind, we can state the following theorem.
\begin{Theorem}\label{thm:general-alquier}
(\cite{alquier-15} and \cite{nips-16})
For any two densities $\pi$ and $\hat{\rho}$ on $\mathcal{F}$, any  $\delta\in(0,1]$, and a real number $\lambda>0$, 
	\begin{align} \label{T:pac}
        \bP \Bigg( \Bigg\{& \omega \in \Omega  \mid  
    E_{f\sim \hat{\rho}} \mathcal{L}^{\ell}_{\y} (f) 
    	 \le \  E_{f\sim \hat{\rho}} \hat{\mathcal{L}}^{\ell}_{\y,N}(f)(\omega)   \\
    &  +\dfrac{1}{\lambda}\!\left[ KL(\hat{\rho} \|\pi) +
    	\ln\dfrac{1}{\delta}
    	+ \Psi_{\ell,\pi}(\lambda,n)  \right]  \Bigg \}\Bigg) > 1-\delta \,, \nonumber
    \end{align}
 where
 $KL(\hat{\rho} \mid \pi)$ is the Kullback-Leibler divergence:
\begin{equation*}
    KL(\hat{\rho} \|\pi) = E_{f\sim\hat{\rho}} \ln \frac{\hat{\rho}(f)}{\pi(f)}\,.
\end{equation*}
and 
 \[ \Psi_{\ell,\pi}(\lambda,N)=\ln E_{f\sim\pi} \bE[e^{\lambda(\mathcal{L}^\ell_\y(f)-\hat{\mathcal{L}}^\ell_{\y,N}(f))}]
\]
\end{Theorem}
Intuitively, in Theorem \ref{thm:general-alquier}, $\pi$ plays the role of
prior distribution density function and $\hat{\rho}$ plays the role of any candidate to posterior distribution on the space of models. 
The Kullback-Leibler divergence measures the distance between the prior and posterior densities. We expect the term $\Psi_{\ell,\pi}(\lambda,N)$ to converge to zero, as $N \rightarrow \infty$. 
In fact, under suitable assumptions on $\y$, this will indeed
be the case.

The numbers $\lambda$ and $\delta$ are tuning parameters.
The parameter $\delta$ regulates the magnitude of the probability via $1-\delta$ and influence the upper bound on the generalization error through a term inversely proportional to
$\delta$. This means that the smaller we would like to get the probability
that the error bound holds higher, the error bound increases too.
Finally, $\lambda$ is a tuning parameter which tells us how much we care about the empirical loss being small, for details see below.

The way to use Theorem \ref{thm:general-alquier} is as follows.
Consider the sampled training data $S=\{\y(t)(\omega)\}_{t=0}^{N}$.
Then, with probability at least $1 -\delta$, the inequality
\begin{align}\label{T:pac1}
& E_{f\sim \hat{\rho}} \mathcal{L}^{\ell}_{\y} (f) 
    	 \le \  E_{f\sim \hat{\rho}} \hat{\mathcal{L}}^{\ell}_{\y,N}(f)(\omega)   \\
    &  +\dfrac{1}{\lambda}\!\left[ KL(\hat{\rho} \|\pi) +
    	\ln\dfrac{1}{\delta}
    	+ \Psi_{\ell,\pi}(\lambda,N)  \right]  \nonumber
\end{align}
holds. If $\delta$ is small enough, it means that we should be unlucky for
\eqref{T:pac1} not to hold.  Our goal should be to find posterior distribution $\hat{\rho}$ such that the right-hand side of \eqref{T:pac1} is small. 
 For such a density $\hat{\rho}$ the average generalization error $E_{f\sim \hat{\rho}} \mathcal{L}^{\ell}_{\y} (f)$
 will be small. This means that under suitable statistical hypothesis, if we either 
\begin{itemize}
\item 
randomly sample a model $f_{*}$ from the probability distribution 
determined by $\hat{\rho}$,  or 
\item we choose the most likely model 
$f_*=\mathrm{arg max}_{f \in \mathcal{F}} \hat{\rho}(f)$, or
\item we choose $f_*$ to be the mean or median of the 
      probability distribution with the density $\hat{\rho}$
\end{itemize}
Then $\mathcal{L}^{\ell}_{\y} (f_*)$ will be small too (at least with a high probability).  That is, we have found a model with a small generalization error.

Since the terms of $\Psi_{\ell,\pi}(\lambda,n)$
and $\ln\dfrac{1}{\delta}$ in \eqref{T:pac1} do not depend on the the posterior density $\hat{\rho}$, choosing a  density $\hat{\rho}$ for which
the right-hand side of \eqref{T:pac1} is small is
equivalent to choosing $\hat{\rho}$ so that
\begin{equation}
    E_{f\sim\hat{\rho}}\hat{\mathcal{L}}^\ell_{\y,N}(f)(\omega)+\frac{1}{\lambda}KL(\hat{\rho}||\pi)\label{eq:minGoal}
\end{equation}
is small.

By tuning the parameters $\lambda$ and the prior
$\pi$ the cost function \eqref{eq:minGoal} will be different, leading to
different choices of $\hat{\rho}$ and model $f_*$, i.e., to different 
learning algorithms. The prior $\pi$ encodes our hypothesis on the model structure. The parameter $\lambda$ regulates the extent we care about fitting training data. For instance, small $\lambda$ mean that we would like the
posterior $\hat{\rho}$ to be close to the prior (we insist on a certain model structure) and we care less about how well
the models fit the training data. In contrast, large $\lambda$ means that
we care less about the model having a certain structure, and more about it
fitting the training data. 

The parameter $\delta$ allows us to regulate the probability that the inequality \eqref{T:pac1} is not true, and hence
the found model will have a potentially arbitrary generalization error. Clearly, the price we pay for making sure that \eqref{T:pac1} is true
with a high probability is that the upper bound on the generalization error
is large. Finally, the error term $\Psi_{\ell,\pi}(\lambda,N)$
depends  on $\lambda$, the prior $\pi$ and $N$, for large enough $N$ we 
expect it to be small. 

In fact, we can present an explicit formula for the density
$\hat{\rho}$ which minimizes \eqref{eq:minGoal}: the unique density $\hat{\rho}^*$ which achieves the minimum of  \eqref{eq:minGoal} is the Gibbs posterior \cite{alquier-15}:
\begin{equation}
    \hat{\rho}^*(f)=\frac{1}{Z_{\y}(\omega)}\pi(f)e^{-\lambda\hat{\mathcal{L}}^\ell_{\y,N}(f)(\omega)} \label{eq:gibbs}, 
\end{equation}
where 
\begin{equation}
\label{eq:gibbs1}
Z_{\y}(\omega)=E_{f \sim \pi} e^{-\lambda\hat{\mathcal{L}}^\ell_{\y,N}(f)(\omega)}
\end{equation}
is the normalization term. 
Note that $\hat{\rho}^{*}(f)$ is a function of the empirical loss 
$\hat{\mathcal{L}}^\ell_{\y}(f)(\omega)$ evaluated for the training data
$S=\{\y(t)(\omega)\}_{t=0}^{N}$. 

If we choose $f_*$ as the most likely model, i.e., 
$f_*=\mathrm{arg max}_{f \in \mathcal{F}} \hat{\rho}^{*}(f)$, then we are solving the optimization problem:
\begin{equation}
    \min_{f \in \mathcal{F}} \Bigg ( \lambda\hat{\mathcal{L}}^\ell_{\y,N}(f)(\omega) - \ln \pi(f) \Bigg) \label{eq:gibbs4}, 
\end{equation}
That is, we choose $f_*$ by minimizing a cost function which includes the
empirical loss and a regularization term. The latter is just the prior
distribution density. This is a further indication that a large variety of
seemingly different learning algorithms can be viewed as special
cases of the PAC-Bayesian framework. 

\begin{Remark}[Using infinite past filters]
 Theorem \ref{thm:general-alquier} and the discussion above
 remains true, if we replace $\hat{\mathcal{L}}_{\y,N}(f)$
 by the alternative definition $V_{\y,N}(f)$ from \eqref{eq:EmpLoss1}. 
\end{Remark}


\section{PAC-Bayesian framework for stochastic LTI systems}
\label{lti:bayesian}
In this paper we will derive a PAC-Bayesian bound for 
learning stochastic LTI systems without inputs. 
More precisely, we will assume that
data generating process $\y$ is the output of a stochastic LTI system, and that
the class of models is that of the class of stochastic LTI systems without inputs. 

This latter point necessitates some explanation. As it was pointed out in the beginning of the paper,
stochastic dynamical systems can be interpreted as generators of stochastic
processes, parametrizations of distributions and predictors, and in a 
sense these three points of view are equivalent. 

In system identification, the objective of learning is formulated as 
finding a model which approximately generates the same observed behavior as
the underlying physical system. 
In Section \ref{sect:pac:learning:gen} we formulate the learning problem as that of finding the best possible predictor for $\y$. In fact, it turns out that for the case of stochastic LTI systems, these two objectives are essentially equivalent, in the sense that there is
a correspondence between best predictor of $\y$ and stochastic LTI systems 
whose output is $\y$. We will come back to this point after having set up the formal mathematical framework.

A stochastic LTI is a system of the form 
\begin{equation}\label{LTIn}
\begin{split}
\x(t+1) & =A\x(t)+B\bv(t)\\
\tilde{\y}(t) & =C\x(t)+D\bv(t)
\end{split}
\end{equation}
where $A \in \mathbb{R}^{n \times n},B \in \mathbb{R}^{n \times m},C \in \mathbb{R}^{p \times n}, D \in \mathbb{R}^{p \times m}$ for $n \ge 0$, $m,p>0$ and $\x$, $\tilde{\y}$ and $\bv$ are zero-mean square-integrable stationary stochastic process with values in $\mathbb{R}^n$, $\mathbb{R}^{p}$, and $\mathbb{R}^m$ respectively. The processes $\x$, $\y$ and $\bv$ are called state, output and noise process, respectively. Furthermore, we require that $A$ is stable (all its eigenvalues are inside the open unit circle) and that for any $t,k \in \mathbb{Z}$, $k \geq 0$, $E[\bv(t)\bv^T(t\!-\! k\!-\! 1)]=0$, $E[\bv(t)\x^T(t-k)]=0$, i.e.,  $\bv(t)$ is white noise and uncorrelated with $\x(t-k)$. 
We identify the system \eqref{LTIn} with the tuple
$(A,B,C,D,\bv)$; note that the state process $\x$ is uniquely defined by
the infinite sum $\x(t)=\sum_{k=1}^{\infty} A^{k-1}B\bv(t-k)$. 
We say that the the stochastic LTI system is a \emph{realization} of the process $\y$, if its output $\tilde{\y}$ coincides with $\y$, i.e., $\y(t)=\tilde{\y}(t)=C\x(t)+D\bv(t)$. 

It is well-known \cite{LindquistBook}, that if $\y$ has a realization by a
LTI system, then it has a realization by a minimal LTI system in forward innovation form, i.e., and LTI system $(A,K,C,I,\e)$ such that 
$I$ is the identity matrix, $(A,K)$ is controllable, $(C,A)$ is observable, and
$\e$ is the \emph{innovation process} of $\y$.

To recall the notion of an innovation process, we need the following. Denote by $\mathcal{H}$ the Hilbert space of zero-mean square-integrable real valued random variables, where the inner product between two random variables $y,z$ is $E[yz]$. The Hilbert space generated by a set $U\subset \mathcal{H}$ is the smallest (w.r.t. set inclusion) closed subspace of $\mathcal{H}$ which contains $U$. We denote by $E_l[z| U]$ the orthogonal projection of $z$ onto $U$, see \cite{LindquistBook} for details. For a zero-mean square-integrable stocastic process $\z(t) \in\mathbb{R}^p$, we let $E_l[\z(t) | U]=[\hat{\z}_1(t),\ldots,\hat{\z}_k(t)]^T$, where $\hat{\z}_i(t)=E_l[\z_i(t) | U]$, $i=1,\ldots,k$. That is, $E_l[\z(t)| U]$ is the random variable with values in $\mathbb{R}^p$ obtained by projecting the coordinates of $\z(t)$ onto $U$. Accordingly, the orthogonality of a multidimensional random variable to a closed subspace in $\mathcal{H}$ is meant element-wise.

	For the stochastic process $\y$ denote by  $\mathcal{H}^{\y}_{t-}$
	the closed subspace of $\mathcal{H}$ generated by the set $\{ \ell^T\y(s) \mid  s \in \mathbb{Z}, s < t,\ell \in \mathbb{R}^p\}$ formed by
	the past values of $\y$. We call the process 
	$$ \e(t):=\y(t)-E_l[\y(t)|\mathcal{H}^{\y}_{t-}], ~ \forall t \in \mathbb{Z} $$
	the \emph{(forward) innovation process} of $\y$, that is, the innovation
	process is the difference between $\y$ and its orthogonal projection 
	to its past. The variance $E[\e^T(t)\e(t)]$ gives us the minimal
	prediction error which is attainable by trying to predict $\y(t)$ as a
	linear function of the past values $\{\y(s)\}_{s< t}$. 

    Minimal realizations in forward innovation form of a process are unique up to
    a linear isomoprhism. We will say that the process $\y$ 
    is \emph{coercive}, if it has a minimal realization
    $(A,K,C,I,\e)$ in forward innovation form such that $A-KC$ is stable. 
    Since all minimal realization of $\y$ in forward innovation form are 
    isomorphic, the stability of $A-KC$ will hold for any minimal realization
    of $\y$. 
    
    \begin{Assumption}
    \label{ass1}
     In the sequel we assume that $\y$ can be realized by a stochastic LTI system and $\y$ is coercive. 
    \end{Assumption}

    Coercivity of $\y$ implies that any minimal realization
    $(A, K, C, I, \e)$ of $\y$ in forward innovation form, 
     defines a linear map $\{\y(s)\}_{s=-\infty}^{t-1} \mapsto \hat{\y}(t)$ as follows:
\begin{equation}
\label{filt:lti}
    \x(t+1)=(A-KC)\x(t)+K\y(t), ~ \hat{\y}(t)=C\x(t). 
\end{equation}
In fact, $\hat{\y}(t)$ is the projection of $\y(t)$ onto the Hilbert-space
generated by the past values $\{\y(s)\}_{s=-\infty}^{t-1}$ of $\y$, i.e., $\hat{\y}(t)=E_l[\y(t) \mid \mathcal{H}^{\y}_{t-}]$.  Note that $\y(t)-\hat{\y}(t)=\e(t)$, i.e., $E[\|\y(t)-\hat{\y}(t)\|_2^2]$ is minimal among all
linear mappings $\{\y(s)\}_{s=-\infty}^{t-1} \mapsto \hat{\y}(t)$, or all such mappings, if $\y$ is Gaussian.  
That is, any minimal realization of $\y$ in forward innovation form yields
an optimal predictor. 

The discussion above prompts us to define the set of models as follows. Let  $\Sigma=(\hat{A},\hat{K},\hat{C})$ be a tuple of matrices such that $\hat{A}$ and $\hat{A}-\hat{K}\hat{C}$ are stable. Define the map
$f_{\Sigma}:\mathcal{Y}^{*} \rightarrow \mathcal{Y}$ as follows:
\begin{equation}
\label{mod:def:lti1}
f_{\Sigma}(y_1,\ldots,y_k)=\sum_{t=1}^{k} \hat{C}(\hat{A}-\hat{K}\hat{C})^{k-t}\hat{K}y_t
\end{equation}
In other words, $\hat{\y}_{f_{\Sigma}}(t \mid s)$ is the output at time
$t$ of the following dynamical system: 
\begin{equation}
\label{filt:lti:fin:theta}
\begin{split}
    & \x_{\Sigma}(t+1 \mid s)=(\hat{A}-\hat{K}\hat{C})\x_{\Sigma}(t \mid s )+\hat{K}\y(t), 
    \x_{\Sigma}(s \mid s)=0
    \\
    & \hat{\y}_{f_{\Sigma}}(t \mid s )=\hat{C}\x_{\Sigma}(t \mid s) 
\end{split}    
\end{equation}
That is $f_{\Sigma}$ is the input-output map of the deterministic LTI system
$(\hat{A}-\hat{K}\hat{C},\hat{K},\hat{C})$.

It is clear that the function $f_{\Sigma}$ is a model in the sense of Definition \ref{model:def}, since
$f_{\Sigma}$ is generated by matrices
$M_k=\hat{C}(\hat{A}-\hat{K}\hat{C})^{k-1}\hat{K}$ which are
Markov parameters of the stable deterministic LTI system
$((\hat{A}-\hat{K}\hat{C}),\hat{K},\hat{C})$, and hence
the sequence $\{M_k\}_{k=1}^{\infty}$ is absolutely summable.

In particular, the limit
$\hat{\y}_{f_{\Sigma}}(t)=\lim_{s \rightarrow -\infty} \hat{\y}_{f_{\Sigma}}(t \mid s)$ exists in the mean square sense uniformly in $t$, and 
\begin{equation}
\label{mod:def:lti2}
\hat{\y}_{f_{\Sigma}}(t)=\sum_{t=1}^{\infty} \hat{C}(\hat{A}-\hat{K}\hat{C})^{k-1}\hat{K}\y(t-k),
\end{equation}
in other words, $\hat{\y}_{f_{\Sigma}}(t)$ is the output of the following
dynamical system
\begin{equation}
\label{filt:lti:fin:theta}
\begin{split}
    & \x_{\Sigma}(t+1)=(\hat{A}-\hat{K}\hat{C})\x_{\Sigma}(t)+\hat{K}\y(t) \\
    & \hat{\y}_{f_{\Sigma}}(t)=\hat{C}\x_{\Sigma}(t) 
\end{split}    
\end{equation}


Recall from Section \ref{sect:pac:learning:gen} that for the formulation of
PAC-Bayesian framework we need to define a parametrized set of models.
In order to do so, we remark that two different tuples $\Sigma$ and
$\Sigma^{'}$ may yield the same model $f_{\Sigma}=f_{\Sigma^{'}}$. 
However, this source of non-uniqueness can be eliminated 
by using minimal tuple. Let us call the tuple
$\Sigma=(\hat{A},\hat{K},\hat{C})$ \emph{minimal}, if
$(\hat{A},\hat{K})$ is a controllable pair, $(\hat{A},\hat{C})$
is an observable pair, and $\hat{A}$, $\hat{A}-\hat{K}\hat{C}$ are stable.
From standard realization theory we can derive the following simple result:
\begin{lemma}\label{pac:lemma:min}
$\hbox{}$
\begin{itemize} 
\item
 If $\Sigma=(\hat{A},\hat{K},\hat{C})$ is a tuple such that $\hat{A}$ and 
 $\hat{A}-\hat{K}\hat{C}$ are stable, then there exists a minimal tuple
 $\Sigma_m$ such that $f_{\Sigma}=f_{\Sigma_m}$.

    \item 
 If $\Sigma=(\hat{A},\hat{K},\hat{C})$ and $\Sigma^{'}=(\hat{A}^{'},\hat{K}^{'},\hat{C}^{'})$ are two minimal tuples, then $f_{\Sigma}=f_{\Sigma^{'}}$ implies that $\Sigma$ and $\Sigma^{'}$
 are isomorphic, i.e. there exists a nonsingular matrix $T$, such that 
 $\hat{A}^{'}=T\hat{A}T^{-1}$, $\hat{K}^{'}=T\hat{K}$, and 
 $\hat{C}^{'}=\hat{C}T^{-1}$.

 \end{itemize}
\end{lemma}


That is, it is enough to consider models which arise from minimal tuples. Moreover, we will consider parametrizations of models which arise via 
parametrizing minimal tuples and in which there are no two isomorphic
minimal tuples. 

More precisely, we consider a family of models $\mathcal{F}$ which satisfies
the following assumption.
\begin{Assumption}
\label{pac:lti:assum}
There exists a measure space  $(\Theta,\mathcal{B},m)$ and a map
$$\mathbf{\Sigma}:\Theta \rightarrow (\mathbb{R}^{n \times n} \times \mathbb{R}^{n \times p}  \times \mathbb{R}^{p \times n})$$
such that the set of models is defined by
     \[ \mathcal{F}=\{ f_{\Sigma(\theta)} \mid \theta \in \Theta\}, \]
and the following holds.
\begin{itemize}
    \item For any $\theta \in \Theta$, $\Sigma(\theta)$ is a minimal tuple,
    \item For any $\theta_1,\theta_2 \in \Theta$, the tuples 
    $\Sigma(\theta_1)$ and $\Sigma(\theta_2)$ are not isomorphic,
     i.e., the map $\Theta \ni \theta \mapsto f_{\Sigma(\theta)} \in \mathcal{F}$ is a
     one-to-one function. 
\item There exists $\theta_0$ such that $\Sigma(\theta_0)=(A_0,K_0,C_0)$
      and $(A_0,K_0,C_0,I,\e)$ is a minimal realization of $\y$ in forward innovation form
\end{itemize}
\end{Assumption}
The intuition behind these assumptions is as follows.
The set $\mathcal{F}$ of models under consideration arises by parametrizing
a subset of tuples which then give rise to models as defined in
\eqref{mod:def:lti1}. The first assumption says that
we parametrize $\mathcal{F}$ in such a manner that
each parameter value corresponds to a minimal tuple. This is not a real constraints, as if for some parameter values the corresponding tuple of
matrices is not a minimal tuple, then by Lemma \ref{pac:lemma:min} we can replace it by a minimal tuple which gives rise to the same model. Hence, such a replacement will not change the set of models $\mathcal{F}$. The second condition means that there are no two parameter values which give rise to the same input-output map. Finally, the third assumption says that $\mathcal{F}$
contains a model arising from a realization of $\y$. 
The latter point is quite useful in terms of relating the learning
problem formulated in Section \ref{sect:pac:learning:gen} with the usual formulation of the system identification problem.

That is, the \emph{model which arises from the realization of $\y$ is the model with the least generalization error.} In fact,
\begin{lemma}
\label{pac:lemma:lti2}
If $$\theta_*=\mathrm{argmin}_{\theta \in \Theta} \mathcal{L}^{\ell}_\y(f_{\theta}), $$
then $(A_0,K_0,C_0,I,\e)$ is a minimal realizaion of $\y$, where $\Sigma(\theta_*)=(A_0,K_0,C_0)$. 
\end{lemma}
\
Lemma \ref{pac:lemma:lti2} means that trying to find a model $\mathcal{F}$ based on sampled data is equivalent to finding the element of $\mathcal{F}$
which corresponds to a realization of $\y$. 
In fact, even if we cannot find the model with the smallest generalization error exactly, this can still be interpreted as finding a stochastic LTI system output of which is close to $\y$.
More precisely, if the parameter set $\Theta$ is a topological space and
if $\Sigma(\theta)$ is continuous in $\theta$ and if we find $\theta \in \Theta$ such that 
$\theta$ is close to $\theta_0$, then the matrices of $\Sigma(\theta)=(\hat{A},\hat{K},\hat{C})$ will be close to
$\Sigma(\theta_0)=(A_0,K_0,C_0)$. It then follows that the output of
the stochastic LTI system $(\hat{A},\hat{K},\hat{C},\e)$ will be close to $\y$.


Now we are able to state the main result on PAC-Bayesian error bound for
stochastic LTI systems.
\begin{Theorem}\label{thm:lti}
Assume that $\y$ and $\mathcal{F}$ satisfy Assumption \ref{ass1}--\ref{pac:lti:assum} and that $\y$ is jointly Gaussian.
For any two densities $\pi$ and $\hat{\rho}$ on $\mathcal{F}$, $\delta\in(0,1]$, and a real number $\lambda>0$, 
\eqref{T:pac} holds, where
\begin{equation}
    \Psi_{\ell,\pi}(\lambda,N)\leq u_{\ell,\pi}(\lambda,N) \label{eq:newPsiBound}
\end{equation}
with
\begin{equation}
    u_{\ell,\pi}(\lambda,N) = \ln\underset{\theta \sim\pi}{E}\begin{bmatrix}\frac{\exp{(\lambda v(\theta))}}{\left (1+\frac{\lambda \rho_N(\theta)}{\frac{N}{2}} \right)^{\frac{pN}{2}}} \end{bmatrix},
\end{equation}
and the number $v(\theta)$ and $\rho_N(\theta)$ are defined as follows:
\begin{equation}
    v(\theta)=\mathcal{L}^{\ell}_\y(f_{\theta})
\end{equation}
and 
\begin{align}
    \rho_N(\theta)=min(eig(\bE \left [\z_{1:N}(f_\theta)\z_{1:N}^T(f_\theta) \right ]))
\end{align}
with
\begin{equation}
\label{pred:error_var}
    \z_{1:N}(f_\theta)=\begin{bmatrix}(\y(1)-\hat{\y}_{f_{\theta}}(1 \mid 0)) \\ \vdots \\ (\y(N)-\hat{\y}_{f_{\theta}}(N \mid 0)) \end{bmatrix}
\end{equation}
\end{Theorem}
\begin{proof}[Proof Theorem \ref{thm:lti}]
The proof follows the same lines as that of 
\cite[Theorem 2]{shalaeva2019improved}, for the sake
of completeness we repeat the basic steps. From Theorem \ref{thm:general-alquier} it follows that
\eqref{T:pac1} holds  with probability at least $1-\delta$.
 Consider a random variable 
  $\z_{i}(f_{\theta})=\y(i)-\hat{y}_{f_{\theta}}(i \mid 0)$.
Just like in
 the proof of \cite[Theorem 2]{shalaeva2019improved}, 
 \begin{equation}
\label{thm:pac-bound-squared2:pf0}
\begin{split}
    &\Psi_{\ell,\pi}(\lambda, N) = \ln E_{f\sim\pi} \bE[e^{\lambda(\mathcal{L}^\ell_\y(f)-\hat{\mathcal{L}}^\ell_{\y,N}(f)}] = \\
    &= \ln E_{f\sim\pi} e^{\lambda \mathcal{L}^\ell_\y(f)}
    \bE\left[e^{-\lambda\hat{\mathcal{L}}^\ell_{\y,N}(f)}\right] = \\
    &=\ln E_{f\sim\pi} e^{\lambda \mathcal{L}^\ell_\y(f)}
    \bE\left[e^{-\frac{\lambda}{N} \sum_{i=1}^{N} \z_i^T(f_{\theta})\z_i(f_{\theta})} \right ]
\end{split}
\end{equation}
 Since $\z_{i}(f_{\theta})$ is a linear combination of $\{\y(s)\}_{s=0}^{i}$, and $\y$ is zero mean Gaussian process, it follows that $\z_{i}(f_{\theta})$ is zero mean Gaussian with covariance 
 $Q_{\z_i(f_{\theta})}=\bE[\z_{i}(f_{\theta})\z_{i}^T(f_{\theta})]$.  
 Notice that the variable
 $\z_{1:N}(f_{\theta})$ from \eqref{pred:error_var}
 is of the form 
 $\z_{1:N}(f_{\theta})=\begin{bmatrix} \z_1(f_{\theta}) \\ \vdots \\ \z_N(f_{\theta}) \end{bmatrix}$. 
  Define  
  \[ S(f_{\theta})=Q_{\z_{1:N}(f_{\theta})}^{-1/2} \z_{1:N}(f_{\theta}) 
\]
 where $Q_{z_{1:N}(f_{\theta})}=\bE[\z_{1:N}(f_{\theta})\z_{1:N}^T(f_{\theta})]$
 and let $S_i$ be the $i$th entry of $S$, i.e.,  $S=\begin{bmatrix} S_1 & \ldots & S_{pN} \end{bmatrix}^T$.
Then from \eqref{thm:pac-bound-squared2:pf0} it follows that 
\begin{equation}
 \label{thm:pac-bound-squared2:pf3}
  \begin{split}
   & \sum_{i=1}^{N} \z_i^T(f_{\theta})\z_{i}(f_{\theta})  = 
   \z_{1:N}(f_{\theta})^T\z_{1:N}(f_{\theta})= \\
    & = \z_{1:N}(f_{\theta})^T Q_{\z_{1:N}(f_{\theta})}^{-1/2} Q_{\z_{1:N}(f_{\theta})} Q_{\z_{1:N}(f_{\theta})}^{-1/2}\z_{1:N}(f_{\theta}) = \\
&  S^T(f_{\theta})Q_{\z_{1:N}}(f_{\theta})S(f_{\theta}) \ge S^T(f_{\theta})S(f_{\theta}) \rho_{N}(\theta)  =  \\
& (\sum_{i=1}^{pN} S_i(f_{\theta})^2 ) \rho_{N}(\theta)
  \end{split}
\end{equation}
It then follows that 
\begin{equation}
 \label{thm:pac-bound-squared2:pf4}
  \begin{split}
  &  e^{-\frac{\lambda}{N} \sum_{i=1}^{N} \z_{i}^T(f_{\theta})\z_i(f_{\theta})}=e^{-\frac{\lambda}{N} \z_{1:N}^T(f_{\theta})\z_{1:N}(f_{\theta})} \le \\
  & e^{-\frac{\lambda}{N} \rho_{N}(\theta)  (\sum_{i=1}^{Np} S_i^2(f_{\theta}))}
  \end{split} 
\end{equation}
 Notice now that $S(f_{\theta})$  is Gaussian and zero mean, with covariance 
  \begin{equation*}
     \begin{split} 
       & E[S(f_{\theta})S^T(f_{\theta})]=Q_{z_{1:N}(f_{\theta}}^{-1/2} E[\z_{1:N}(f_{\theta})\z_{1:N}^T(f_{\theta})] Q_{z_{1:N}(f_{\theta})}^{-1/2} = \\
      & = Q_{z_{1:N}(f_{\theta})}^{-1/2}Q_{z_{1:N}(f_{\theta})}Q_{z_{1:N}(f_{\theta})}^{-1/2}=I
    \end{split}
  \end{equation*}
 That is,the random variables $S_i(f_{\theta})$ are normally distributed and $S_i(f_{\theta}),S_j(f_{\theta})$ are independent, and therefore $\sum_{i=1}^{Np} S_i^2(f_{\theta})$ has $\chi^2$ distribution. Hence,
\[ 
    \bE[e^{-\frac{\lambda \rho_{N}(\theta)}{N}  (\sum_{i=1}^{Np} S_i^2(f_{\theta})}]=\frac{1}{(1+\frac{\lambda\rho_{N}(\theta)}{\frac{N}{2}})^{\frac{pN}{2}}}
\]
Combining this with \eqref{thm:pac-bound-squared2:pf4} and \eqref{thm:pac-bound-squared2:pf0} implies the statement of
the theorem. 
\end{proof}

\subsection{Discussion of the error bound}
\label{sect:discussBound}
Below we will make a number of remarks concerning the error
bound \eqref{eq:newPsiBound}.

\paragraph{Asymptotic convergence  of $\Psi_{\ell,\pi}(\lambda,N)$}
\begin{lemma}
\label{Psi:convergence}
For some subsequence $N_k$, $\lim_{k \rightarrow \infty} N_k=\infty$, $\lim_{k \rightarrow \infty} \Psi_{\ell,\pi}(\lambda,N_k)=0$. 
\end{lemma}

This result indicates that $\Psi_{\ell,\pi}(\lambda,N)$ cannot 
converge to infinity or to a non-zero constant. This also indicates
the possibility that $\Psi_{\ell,\pi}(\lambda,N)$ might converge
to zero as $N \rightarrow \infty$.

\paragraph{Behavior of the upper bound as $N \rightarrow \infty$, conservativity of the result}
Lemma \eqref{Psi:convergence} also  suggests that the upper bound \eqref{eq:newPsiBound} is a conservative one, as 
the upper bound does not converge to $0$. In fact, the following holds.

\begin{lemma}
\label{upperBound:lem1}
The sequence $\rho_{N}(\theta)$ is decreasing and 
$\lim_{N \rightarrow \infty} \rho_N(\theta)=\rho^{*}(\theta)  > \mu_{*} > 0$.
In particular,  if
the spectral density of $\y$ is denoted by $\Phi_{\y}$,
then $\mu_{*} > 0$ can be taken such that
$\Phi_{y}(e^{i\theta}) > \mu_{*}I$ for all $\theta \in [-\pi,\pi]$.
In particular, in this case $v(\theta) > p\mu_*$. 
\end{lemma}
Lemma \ref{upperBound:lem1} says that 
$\rho_{N}(\theta)$ is bounded from below. 
\begin{lemma}
\label{upperBound:lem2}
$$\ln \underset{\theta \sim\pi}{E}[e^{\lambda (v(\theta)-p \rho_0)}] \le  \lim_{N \rightarrow \infty} u_{\ell,\pi}(\lambda,N) \le \ln \underset{\theta \sim\pi}{E}[e^{\lambda (v(\theta)-p\mu_*)}],$$
where 
$\mu_*$ is as in Lemma \ref{upperBound:lem1}, and 
$\rho_0$ is the smallest eigenvalue of
$\bE[\y(t)\y^{T}(t)]$
\end{lemma}

Lemma \ref{upperBound:lem2} together with Lemma 
\ref{Psi:convergence} indicates that the upper bound of
Theorem \ref{thm:lti} might be quite conservative:
while $\Psi_{\ell,\pi}(\lambda,N)$ converges to zero for
a suitable strictly increasing sequence of  integers $N_k$,
the upper bound 
$u_{\ell,\pi}(\lambda,N)$ is bounded from below in the limit. 
Note that $\y(t)$ is an i.i.d process with variance $\sigma I_p$, then $\rho_0=\sigma=\mu_*$ and the upper and lower bounds of Lemma \ref{upperBound:lem2} coincide. 

The problem of finding a tight PAC-Bayesian upper bound 
for LTI systems seems to be a challenging problem. In fact,
the first PAC-Bayesian upper bound for linear regression
\cite{nips-16} with i.i.d. data had the same drawback, 
namely, the part of the upper bound on the term which corresponds to
$\Psi_{\ell,\pi}(\lambda,N)$ did not converge to zero
as $N \rightarrow \infty$. This drawback for eliminated in
\cite{shalaeva2019improved} for linear regression with i.i.d data, but not for linear regression with non i.i.d. data. 
In the light of this development, the conservativity of
our error bound is not surprising. 

\paragraph*{Upper bound not involving $N$ and expectation over the prior}
Finally, we can formulate a non-asymptotic upper bound similar to that of Lemma \ref{upperBound:lem2}
\begin{lemma}
\label{upperBound:lem3}
\[ 
\begin{split}
& u_{\ell,\pi}(\lambda,N) \le \ln \underset{\theta \sim\pi}{E}[e^{\lambda (v(\theta)-\frac{p\mu_* pN*0.5}{\lambda p\mu_* + 0.5pN})}]  \le \\
& \le \underset{\theta \sim\pi}{E}[e^{\lambda (v(\theta)}] - \frac{\lambda p\mu_* 0.5}{\lambda p\mu_* + 0.5}
\end{split}
\]
where 
$\mu_*$ is as in Lemma \ref{upperBound:lem1}
\end{lemma}

The result of Lemma \ref{upperBound:lem2} -- \ref{upperBound:lem3} could be used
to derive further PAC-Bayesian error bounds. In fact, 
\[
   \lim_{N \rightarrow \infty} \ln \underset{\theta \sim\pi}{E}[e^{\lambda (v(\theta)-p\mu_*)}]=
    \ln \underset{\theta \sim\pi}{E}[e^{\lambda (v(\theta))}]
    - p \mu_{*},
\]
so any upper bound on $\ln \underset{\theta \sim\pi}{E}[e^{\lambda (v(\theta))}]$
yields an upper bound on the limit of $u_{\ell,\pi}(\lambda,N)$.
Alternatively, we can also use the non-asymptotic 
bound from Lemma \ref{upperBound:lem3}.
In particular, if 
\[ K_1 < v(\theta) < K_2 \]
for every $\theta$, then
\begin{eqnarray}
    & K_1 - p \rho_0 \le \lim_{N \rightarrow \infty} u_{\ell,\pi}(\lambda,N)  \le K_2 -p \mu_{*} 
    \label{disc:bound:eq1}
    \\
    &  u_{\ell,\pi}(\lambda,N)  \le K_2 -\frac{0.5p \mu_{*}}{p\mu_*+0.5} 
\end{eqnarray}

The latter bounds do not depend on 
the prior, but only the bounds $K_1,K_2$ and 
the constants $\rho_0 > \mu_*$, which, in turn,
represent some mild assumptions on the distribution of $\y$.
Furthermore, it is easy to see from the discussion
in Section \ref{sec:Psi} that 
$v(\theta)$ is a continuous function of $\theta$, hence,
for a compact parameter set $\Theta$ the
bounds $K_1,K_2$ will exist and they could be estimated based
on the parametrization. That is, 
\eqref{disc:bound:eq1} could readily be applied in this case.

In turn, \eqref{disc:bound:eq1} together with Theorem \ref{thm:lti} and \eqref{T:pac} yields bounds for the 
average generalization error of the posterior for a 
certain choice of $N$ and $\lambda$.

\paragraph{Using the alternative definition of empirical loss involving infinite past filters}
Recall from Remark \ref{rem:emploss} the alternative definition
of empirical loss, which uses the infinite past
prediction $\hyf(t)$, instead of the finite past prediction
$\hyf(t \mid 0)$. We could have stated Theorem \ref{thm:general-alquier} with $V_{\y,N}(f_{\theta})$ instead  of $\hat{\mathcal{L}}_{\y,N}(f_{\theta})$  and we could have stated Theorem \ref{thm:lti} too, with the difference that
$\z_{1:N}(f_{\theta})=\begin{bmatrix}(\y(1)-\hat{\y}_{f_{\theta}}(1)) \\ \vdots \\ (\y(N)-\hat{\y}_{f_{\theta}}(N)) \end{bmatrix}$.
In particular, with these changes in mind the discussion above
still holds, but the proofs simplify a lot. 
In particular, $E[\z_{1:N}(f_{\theta})\z_{1:N}(f_{\theta})^T]$
has the Toepliz-matrix structure, as $\y(t)-\hat{f}_{f_{\theta}}(t)$ is a wide-sense stationary process.

\paragraph{PAC-Bayesian error bounds and system identification}
As it was pointed out in Subsection \ref{sect:pac:learning:gen}, the derived error bound together with the Gibbs posterior distribution 
\eqref{eq:gibbs} leads to learning algorithms.
In particular, choosing a model $f$ which maximizes the likelihood of the Gibbs posterior amounts to minimizing 
the cost function \eqref{eq:gibbs4}. In particular, for 
$\lambda=1$ and $\pi$ being the uniform distribution, 
the cost function \eqref{eq:gibbs4} becomes the empirical loss.
The algorithm which calculates the model minimizing 
\eqref{eq:gibbs4} then corresponds to prediction error minimization \cite{LjungBook}. Note that in the theoretical analysis of the prediction error approach, usually the prediction error \eqref{eq:EmpLoss1} using the infinite past is used. As it was pointed out in Remark \ref{rem:emploss}, for large enough $N$, the latter prediction error is close to the empirical loss with a high probability. 
In fact, in the practical implementation of prediction error method usually the empirical error \eqref{eq:EmpLoss} is used. 
The various other choices of $\lambda$ and the prior $\pi$
when minimizing \eqref{eq:gibbs4} amount to minimizing the prediction error with a regularization term. Other methods for sampling from the posterior will lead to different system identification algorithms. Investigating the relationship between the thus arising algorithms and the existing system identification algorithms remains a topic of future research. 

\subsection{Computing $u_{\ell,\pi}(\lambda,N)$}
\label{sec:Psi}
The value of $u_{\ell,\pi}(\lambda,N)$ can be computed using sequential Monte-Carlo \cite{MonteCarloMethods}. In details, sample a set of parameters $\Theta=\{\theta_i\}_{i=1}^{N_f}$, according to the prior distribution $\theta_i\sim\pi(\theta)$. Then for large enough $N_f$ the right-hand side of \eqref{eq:newPsiBound} can be approximated by
\begin{align}
    \ln \left ( \frac{1}{N_f}\sum_{i=1}^{N_f} \frac{\exp{(\lambda v(\theta_i))}}{\left (1+\frac{\lambda\rho^*(\theta_i)}{\frac{N}{2}} \right)^{\frac{N}{2}}} \right )
\end{align}
For every $\theta_i$ drawn from the prior distribution $\pi(\theta)$, the quantities $v(\theta_i)$ and $\rho^{*}(\theta_i)$ can be calculated as follows. 
Given the system in forward innovation form $(A_0,K_0,C_0)$ and a predictor $f_{\theta_i}$ with 
$$
\Sigma(\theta_i)
=(\hat{A},\hat{K},\hat{C})
=(\hat{A}(\theta_i),\hat{K}(\theta_i),\hat{C}(\theta_i)).
$$ 
Define a system $(A_e,K_e,C_e,I,\e)$, whose output is the error $\z(t)=\y(t)-\hat{\y}_{f_{\theta_i}}(t \mid 0)$,

\begin{equation}
\begin{split}
\underset{\tilde{\x}(t+1)}{\underbrace{\begin{bmatrix} \x(t+1) \\ \hat{\x}(t+1) \end{bmatrix}}}&= \underset{A_e}{\underbrace{\begin{bmatrix} A_0 & 0 \\ \hat{K}C_0 & \hat{A} - \hat{K}\hat{C}  \end{bmatrix}}}\underset{\tilde{\x}(t)}{\underbrace{\begin{bmatrix*}[l] \x(t) \\ \hat{\x}(t) \end{bmatrix*}}}+\underset{K_e}{\underbrace{\begin{bmatrix*}[l] K_0 \\ \hat{K} \end{bmatrix*}}}\e(t), \label{eq:ErrSysSt}\\
\z(t) &= \underset{C_e}{\underbrace{\begin{bmatrix} C_0 & -\hat{C} \end{bmatrix}}}\begin{bmatrix*}[l] \x(t) \\ \hat{\x}(t) \end{bmatrix*} + \e(t),\\
\begin{bmatrix} \x(0) \\ \hat{\x}(0) \end{bmatrix}&=\begin{bmatrix} \sum_{k=1}^{\infty} A_0^{k-1}K_0\e(-k) \\ 0 \end{bmatrix},\\
P_{0}&=\begin{bmatrix}P_x & 0\\0 &0_{n\times n} \end{bmatrix},
\end{split}
\end{equation}
where $P_x=E[\x(0)\x(0)^T]$ is the steady-state covariance, which satisfies the Lyapunov equation
\begin{equation}
    P_x=A_0P_xA_0^T+K_0Q_eK_0^T,
\end{equation}
with $Q_e=E[\e(t)\e(t)^T]$. In practice, the numerical value of the initial condition $\x(0)$ is chosen based on the particular application. Then $v(\theta)$ corresponds to the second moment of the random variable $\z(t)$
\begin{align}
    v(\theta)=E\left [\sum\z_i^2 \right ]=trace(E[\z(t)\z(t)^T]=\nonumber \\
    =trace(C_eP_eC_e^T+Q_e). 
\end{align}
The covariance matrix $P_e$ is the solution to Lyapunov equation
\begin{equation}
    P_e=A_eP_eA_e^T+K_eQ_eK_e^T.
\end{equation}
Now, define the joint covariance matrix $Q_{z,N}$ of the random variable $\z_{1:N} = [\z(1),\dots, \z(N)]^T$ as
\begin{equation}
    Q_{\z,N} = E[\z_{1:N}\z_{1:N}^T],\label{eq:Qzn}
\end{equation}
i.e., the $ij$th $p\!\times\!p$ block matrix element $(Q_{\z,N})_{ij}$ of $Q_{\z,N}$ is $E[\z(i)\z(j)^T]$.
Expanding \eqref{eq:Qzn} with
{\footnotesize
\begin{equation}
     \z_{1:N}= \begin{bmatrix} C_e\\C_eA_e\\ \vdots \\ C_eA_e^N \end{bmatrix} \tilde{x}_e(0) + \begin{bmatrix}I&0&0&0\\C_eA_e^0K_e&I&0&0\\ \vdots& & \ddots&0\\C_eA_e^{N-1}K_e& \dots & C_eA_e^0K_e & I \end{bmatrix} \begin{bmatrix} \e(0)\\ \vdots \\ \e(N) \end{bmatrix},
\end{equation}
}
yields
\begin{align}
(Q_{\z,N})_{ij}=\begin{cases}
C_eP_{i-1}C_e^T+Q_e & \text{ if } i=j \\ 
(C_eP_{i-1}A_e^T+Q_eK_e^T)(A_e^{j-i-1})^TC_e^T & \text{ if } i<j \\ 
C_e(A_e^{i-j-1})(A_eP_{j-1}C_e^T+K_eQ_e) & \text{ if } i>j 
\end{cases}
\end{align}
with
\begin{align}
    P_{i+1}=A_eP_{i}A_e^T+K_eQ_eK_e^T.
\end{align}
The state covariances $P_i$, are computed under the assumption that the stochastic process $\y(t)$ has non-zero state and covariance at $t=0$, encoding the infinite past.

\section{Illustrative Example}
\label{sect:num}
Consider a stochastic process $\y(t)$, defined by the LTI system
\begin{equation}
    \begin{split}
        \x(t+1)&=\begin{bmatrix}0.5 & 1\\0 & 0.5 \end{bmatrix}\x(t)+\begin{bmatrix} 1 & 0\\0 & 1 \end{bmatrix}\bv(t)\\
        \y(t)&=\begin{bmatrix} 1 & 0 \end{bmatrix}\x(t)+\begin{bmatrix} 1 & 0 \end{bmatrix}\bv(t)\\
        \bv(t)&\sim \mathcal{N}(0,1) \label{eq:ExSys}
    \end{split}
\end{equation}
Given a realisation of $\y(t)$ of $N=100$ samples, we wish to estimate a parameter $\theta$ of the system
\begin{equation}
    \begin{split}
        \hat{\x}(t+1)&=\begin{bmatrix}0.5 & 1\\0 & 0.5+\theta \end{bmatrix}\hat{\x}(t)+\begin{bmatrix} 1 & 0\\0 & 1 \end{bmatrix}\bv(t)\\
        \hat{\y}(t)&=\begin{bmatrix} 1 & 0 \end{bmatrix}\hat{\x}(t)+\begin{bmatrix} 1 & 0 \end{bmatrix}\bv(t)\\
        \bv(t)&\sim \mathcal{N}(0,1) \label{eq:ExPredictor}
    \end{split}
\end{equation}
\textbf{Remark} For illustration purposes, $\theta$ represents the deviation from System \eqref{eq:ExSys}, such that $\theta=0$ would yield the smallest generalisation error. \\

A forward innovation form of a known LTI system in the form \eqref{LTIn} can be found by solving algebraic Riccati equations, described in \cite[Section~6.9]{LindquistBook}. Yielding a predictor
\begin{equation}
    \begin{split}
        \x(t+1)&=\underset{A_0}{\underbrace{\begin{bmatrix}0.5 & 1\\0 & 0.5 \end{bmatrix}}}\x(t)+\underset{K_0}{\underbrace{\begin{bmatrix} 1\\0.12 \end{bmatrix}}}e(t)\\
        \y(t)&=\underset{C_0}{\underbrace{\begin{bmatrix} 1 & 0 \end{bmatrix}}}\x(t)+\begin{bmatrix} 1 \end{bmatrix}e(t) \label{eq:ExFIFSys}
    \end{split}
\end{equation}

From System description \eqref{eq:ExFIFSys} in forward innovation form, and parameterisation described in the example description \eqref{eq:ExPredictor}, the hypothesis set can be defined as
\begin{align}
    \begin{split}
    \mathcal{F}=\left \{f_\theta=(A(\theta),K_0,C_0) \mid A(\theta)=A_0+\begin{bmatrix}0&0\\0& \theta \end{bmatrix},\right .\\
    \theta\in\reals,~\rho(A(\theta))<1,~\rho(A(\theta)-K_0C_0)<1 \label{eq:HypothesisSet}\\
    det(\begin{bmatrix} K_0 & A(\theta)K_0\end{bmatrix})\neq 0,det\left(\begin{bmatrix}C_0\\ C_0A(\theta) \end{bmatrix}\right)\neq 0\left. \right \}
    \end{split}
\end{align}
where $\rho(\cdot)$ denotes the spectral radius.

Equipped with the hypothesis set \eqref{eq:HypothesisSet}, a realisation of stochastic process $S=\{\y(t)(\omega)\}_{t=0}^{100}$ of $N=100$ samples of System \eqref{eq:ExSys}, and a prior uniform distribution over the hypothesis set as $\pi\sim\mathcal{U}(-1.5,0.5)$.
The posterior distribution $\hat{\rho}(\theta)$ and the upperbound on the expectation of generalised loss $E_{f\sim \hat{\rho}} \mathcal{L}^{\ell}_{\y} (f)$, can be computed via sequential Monte-Carlo.

First sample a set of parameters
\begin{align}
    \Theta=\{\theta_i\}_{i=1}^{N_f},\quad &\theta_i\sim\pi
\end{align}
and construct a predictor $f_{\theta_i},\: \forall \theta_i\in\Theta$  with $$\Sigma(\theta_i)=(\hat{A},\hat{K},\hat{C})=(A(\theta_i),K_0,C_0)$$
For each $f_{\theta_i}$ compute empirical predictions by iterating
\begin{equation}
\begin{split}
    &\hat{x}(t+1)=(\hat{A}-\hat{K}\hat{C})\hat{x}(t) + \hat{K}\y(t)(\omega),\\
    &\hat{y}(t)=\hat{C}x(t),\\
    &\hat{x}(0)=0.
\end{split}
\end{equation}
From empirical predictions compute empirical loss
\begin{equation}
    \hat{\mathcal{L}}^\ell_{\y,N}(f_{\theta_i})=\frac{1}{N}\sum_{t=1}^{N} (\y(t)(\omega)-\hat{y}(t))^T(\y(t)(\omega)-\hat{y}(t)). \forall i\label{eq:empiricalLossEx}
\end{equation}
Figure \ref{fig:empLoss} shows empirical loss over the parameter set $\Theta$. Note that $\theta=0$ does not yield the smallest empirical loss, since we are using a finite number of samples.
\begin{figure}
\centering
\includegraphics[width=3in]{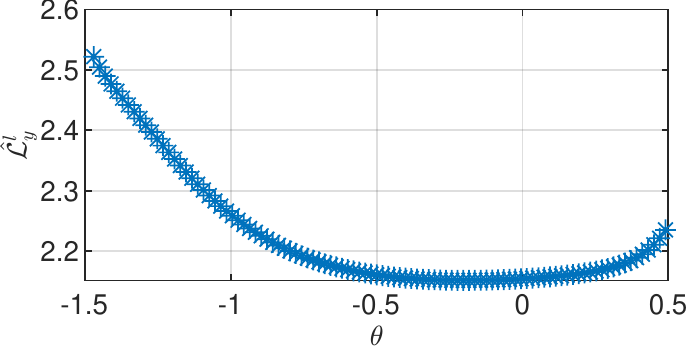} 
\caption{Empirical Loss over parameter set $\theta$}
\label{fig:empLoss}
\end{figure}
From \eqref{eq:empiricalLossEx} and prior distribution $\pi$, the Gibbs posterior distribution $\hat{\rho}$ \eqref{eq:gibbs} can be approximated, and is shown in Figure \ref{fig:gibbs}. In this example, the true system is known and therefore $\lambda$ could be chosen to maximise the likelihood of $\theta=0$. However, in an application, one would need to consider $\lambda$ carefully. 

\begin{figure}
\centering
\includegraphics[width=0.75\linewidth]{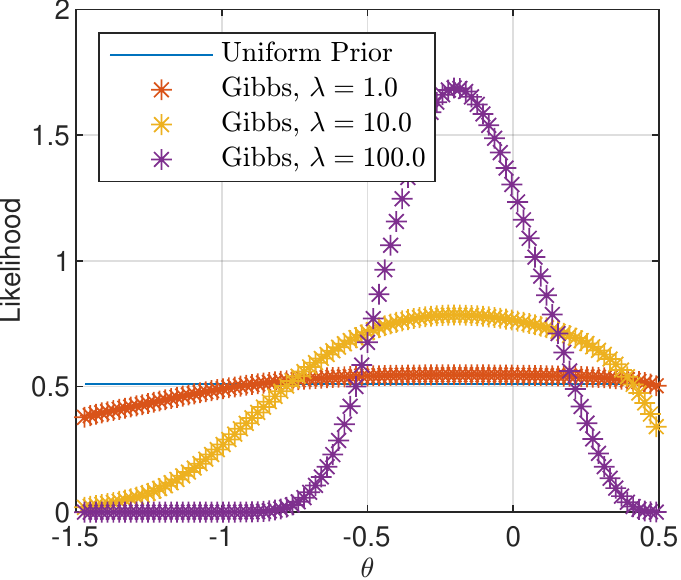} 
\caption{Various Gibbs distributions with varying $\lambda$}
\label{fig:gibbs}
\end{figure}
Now, the computed posterior distribution can be used to calculate the expectation over empirical losses. Normally, one would have to repeat the sequential Monte-Carlo process, while sampling $\theta\sim\hat{\rho}$, and take the sample-mean as an approximation of the expectation. In this simplified example, the expectation can be approximated by approximating an integral
\begin{equation}
    E_{f\sim \hat{\rho}}\hat{\mathcal{L}}^\ell_{\y,N}(f_{\theta})\approx\sum_{i=1}^{N_f-1}\hat{\mathcal{L}}^\ell_{\y,N}(f_{\theta_i})\hat{\rho}(\theta_i)(\theta_{i+1}-\theta_i)
\end{equation}
The KL divergence can similarly be approximated as
\begin{equation}
    KL(\hat{\rho}||\pi)\approx \sum_{i=1}^{N_f-1}\ln \left (\frac{\hat{\rho}(\theta_i)}{\pi(\theta_i)}\right )\hat{\rho}(\theta_i)(\theta_{i+1}-\theta_i)
\end{equation}
Lastly $u_{\ell,\pi}(\lambda,N)$ can be computed according to Section \ref{sec:Psi}. \\
Now we have everything necessary to compute the PAC-bayesian upper bound, the right hand side of \eqref{T:pac1}. For $\lambda=100$, the upper-bound ends up being $4.48$, while $E_{f\sim\hat{\rho}}\mathcal{L}^\ell_\y\approx2.03$. Figure \ref{fig:upperBoundContour} showcases the upper bound for various number of training samples $N$ and real values $\lambda>0$. Recall, that in Gibbs distribution \eqref{eq:gibbs}, $\lambda$ plays a role of weighting the importance of data versus the importance of the prior distribution, larger $\lambda$ will minimise $E_{f\sim\rho}\hat{\mathcal{L}}_\y^\ell$, however it will increase $KL(\rho\|\pi)$. Therefore, it is natural that there exists $\lambda$ that compromises empirical data fit, with prior assumptions. For this specific example, realisation and prior, $\lambda=5$ yields the smallest upper bound for all $N$.

\begin{figure}[h]
\centering
\includegraphics[width=3in]{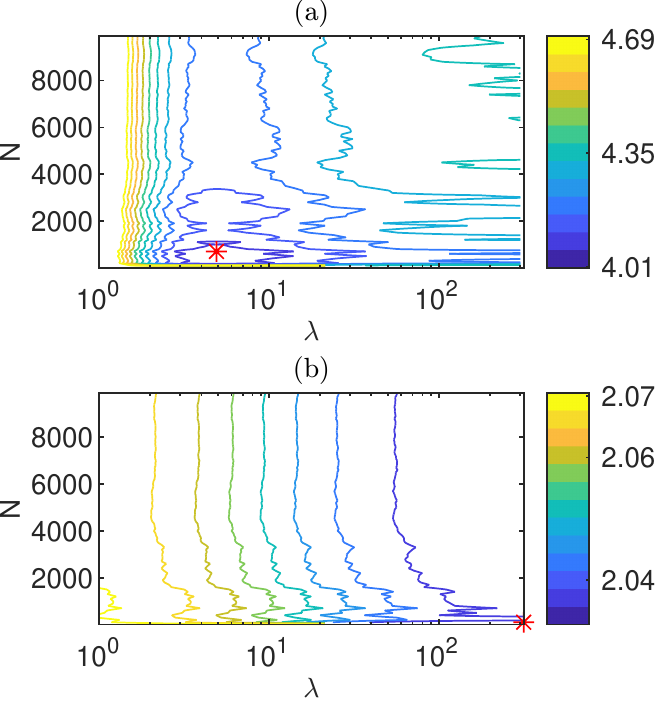} 
\caption{ (a) shows the PAC-bayesian upper bound (right hand side of \eqref{T:pac1}) (b) shows the expectation of the generalisation error over posterior distribution, which depends on $(N,\lambda)$, the minimal value over all $N$ and $\lambda$ is denoted by {\color{red}$*$}}
\label{fig:upperBoundContour}
\end{figure}



\section{Conclusions and future work}
\label{sect:concl}
In this paper we have formulated a non-trivial error PAC-Bayesian error bound for autonomous stochastic LTI models. This error bounds gives a non-asymptotic error bound on the generalization error.  Future work will be directed towards 
extending this bound first to stochastic LTI models with inputs, and then to more complex models such as hybrid models and recurrent neural network. Another research direction would be to understand the effect of the choice of prior and the constant $\lambda$ on the resulting learning algorithms and comparing it to various existing learning algorithms.

	\bibliographystyle{plain}
\bibliography{bib}

\appendix 
\section{Proofs}
\begin{proof}[Proof Lemma \ref{lem:rem:emploss}]
 First we show that $\lim_{t \rightarrow \infty} (\hyf(t\mid 0)-\hyf(t))=0$ in the mean square sense.
 Indeed, for any $\tau \in \mathbb{Z}$
 $\lim_{s \rightarrow -\infty} (\hyf(\tau \mid s)-\hyf(\tau))=0$
 in the mean square sense. 
 Notice that 
 $(\hyf(0 \mid s)-\hyf(0))=\sum_{k=1}^{\infty} M_{k-s}\y(s-k)$,
 where the infinite sum is absolutely convergent in the mean square sense. 
 Define the sequence $b_N=\bE[\|\sum_{k=1}^{\infty} M_{k+N} \y(-N-k)\|_2^2]$. Since $\bE[\|(\hyf(0 \mid s)-\hyf(0))\|_2^2]=b_{-s}$, 
 if follows that the sequence $b_{N}$
 converges to zero as $N$ converges to $\infty$.
 Since $\y$ is stationary, it follows that
 for any $N \ge 0$, any $\tau \in \mathbb{Z}$
$b_N=\bE[\|\sum_{k=1}^{\infty} M_{k+N} \y(-N-k)\|_2^2]=\bE[\|\sum_{k=1}^{\infty} M_{k+N} \y(\tau-k-N)\|_2^2]$.
In particular, by choosing $\tau=N$, it follows that
$b_N=\bE[\|\sum_{k=1}^{\infty} M_{k+N} \y(-k)\|_2^2]$.
 Finally, notice that 
 $\hyf(t\mid 0)-\hyf(t)=\sum_{k=1}^{N} M_{k+t} y(-k)$, and hence
 $\bE=[\|\hyf(t\mid 0)-\hyf(t)\|^2]=b_{t}$. Since 
 $\lim_{t \rightarrow \infty} b_t=0$, it follows that 
 $\lim_{t \rightarrow \infty} (\hyf(t\mid 0)-\hyf(t))=0$ in the mean square sense, as claimed.
 Since $\lim_{t \rightarrow \infty} (\hyf(t\mid 0)-\hyf(t))=0$ in the mean square sense, 
 $\lim_{t \rightarrow \infty} a_t=0$, where  $a_t=\bE[|\ell(\hyf(t\mid 0),\y(t))-\ell(\hyf(t),\y(t))|])$.
 It is easy to see that 
 $\hat{\mathcal{L}}_{\y,N}(f)-V_{\y,N}(f)=\frac{1}{N} \sum_{t=1}^{N} \left(\ell(\hyf(t\mid 0),\y(t))-\ell(\hyf(t),\y(t)) \right)$, and hence
 $\bE[|\hat{\mathcal{L}}_{\y,N}(f)-V_{\y,N}(f)|] \le \frac{1}{N} \sum_{t=1}^{N} a_t$. 
 Since $a_t$ converges to $0$ as $t \rightarrow \infty$, by the well-known property of
 Cesaro-means $\lim_{N \rightarrow \infty}\frac{1}{N} \sum_{t=1}^{N} a_t=0$. 
 \end{proof}
\begin{proof}[Proof of Lemma \ref{pac:lemma:min}]
 Note that if $\Sigma=(\hat{A},\hat{K},\hat{C})$ is a minimal tuple, then
$(\hat{A}-\hat{K}\hat{C},\hat{K},\hat{C})$ is a minimal determinstic LTI, since controllability of $(\hat{A},\hat{K})$ implies that
$(\hat{A}-\hat{K}\hat{C},\hat{K})$ is controllable. 
Similarly, if $\Sigma^{'}$ is a minimal tuple, then 
$(\hat{A}^{'}-\hat{K}^{'}\hat{C}^{'},\hat{K}^{'},\hat{C}^{'})$ is a minimal dimensional  deterministic LTI system. 
Notice that $f_{\Sigma}$ and $f_{\Sigma^{'}}$ are the input-output maps of
the deterministic LTI systems $S_1=(\hat{A}-\hat{K}\hat{C},\hat{K},\hat{C})$
and $S_2=(\hat{A}^{'}-\hat{K}^{'}\hat{C}^{'},\hat{K}^{'},\hat{C}^{'})$
respectively. Hence, from standard realization theory  it follows that
as both $S_1$ and $S_2$ are minimal, $f_{\Sigma}=f_{\Sigma^{'}}$
implies that there exists a nonsingular matrix $T$ such that
$\hat{A}^{'}-\hat{K}^{'}\hat{C}^{'}=T(\hat{A}-\hat{K}\hat{C})T^{-1}$, $\hat{K}^{'}=T\hat{K}$, and 
 $\hat{C}^{'}=\hat{C}T^{-1}$. From this the second statement of
 the lemma follows.

 In order to prove the first part, consider the the deterministic
 LTI system $S_1=(\hat{A}-\hat{K}\hat{C},\hat{K},\hat{C})$ and apply to it
 the standard minimization algorithm. Let the matrices of the resulting
 system be $S_2=(\hat{F}_m,\hat{K}_m, \hat{C}_m)$ and define 
 $\hat{A}_m=\hat{F}_m+\hat{K}_m\hat{C}_m$. Notice that 
 $\hat{F}_m$ is stable, since $\hat{A}-\hat{K}\hat{C}$ was stable and
 minimization preserves stability. It is not difficult to see that
 minimization preserves stability of $\hat{A}$ too, i.e., 
 $\hat{A}_m$ will be stable too. Hence, the tuple $\Sigma_m=(\hat{A}_m,\hat{K}_m,\hat{C}_m)$ is a minimal tuple. 
 Hence, since the input-output maps
 of $S_1$ and $S_2$ coincide, and the input-output map of the former
 is $f_{\Sigma}$ and of the later is $f_{\Sigma_m}$, and
 hence $f_{\Sigma}=f_{\Sigma_m}$.
\end{proof}
\begin{proof}[Proof of Lemma \ref{pac:lemma:lti2}]
 Indeed, assume that  $\Sigma(\theta_0)=(A_0,K_0,C_0)$ is a minimal
tuple such that $(A_0,K_0,C_0,I,\e)$ is a minimal realization of $\y$.
It is not difficult to see that in this case, 
$\hat{\y}_{f_{\Sigma_0}}(t)=E_l[\y(t) \mid \mathcal{H}^{\y}_{t-}]$, 
i.e., $\hat{\y}_{f_{\Sigma_0}}(t)$ is the orthogonal projection of $\y(t)$ onto its own past.
Hence, $\e(t)=\hat{\y}_{f_{\Sigma_0}}(t)-\y(t)$, $\ell(\hat{\y}_{f_{\Sigma_0}}(t),\y(t))=\|\hat{\y}_{f_{\Sigma_0}}(t)-\y(t)\|_2^2$, and therefore
\begin{equation}
\label{pac:lti2}
   E[\ell(\hat{\y}_{f_{\Sigma_0}}(t),\y(t))]=
   E[[\|\e(t)\|^2_2]
\end{equation}
It is easy to see that for any other element of $f \in \mathcal{F}$, 
the entries of $\hat{\y}_{f}(t \mid s)$ and $\hat{\y}_{f}(t)$ belong to the closed subspace $\mathcal{H}^{\y}_{t-}$ generated by the past of $\y$.
By the properties of orthogonal projections, it then follows that
\begin{equation}
\label{pac:lti3}
\begin{split}
& \mathcal{L}^{\ell}(f) = E[\ell(\hat{\y}_{f}(t),\y(t))] = \\
& =E[\|\hat{\y}_{f}(t)-\y(t)\|^2_2] \le E[[\|\e(t)\|^2_2]= \\
& =E[\ell(\hat{\y}_{f_{\Sigma(\theta_0)}}(t),\y(t))]=\mathcal{L}^{\ell}(f_{\Sigma(\theta_0)}))
\end{split}
\end{equation}
and the equality holds only if $\hat{\y}_{f}(t)=\hat{\y}_{f_{\Sigma(\theta_0)}}(t)$. The latter
means that $\y(t)-\hyf(t)=\e(t)$. This means that if 
$f=f_{\Sigma}$ for a minimal tuple $\Sigma=(\hat{A},\hat{K},\hat{C})$, then
by using \eqref{filt:lti:fin:theta} and substituting 
$\y(t)-\hat{C}\x_{\Sigma}(t)=\y(t)-\hyf(t)=\e(t)$ into the first equation of \eqref{filt:lti:fin:theta}, it follows that
\begin{align*}
    & x_{\Sigma}(t+1)=\hat{A}\x_{\Sigma}(t)+\hat{K}\e(t) \\
    & \y(t)=\hat{C}\x_{\Sigma}(t)+(\y(t)-\hat{C}\x_{\Sigma}(t))=\hat{C}\x_{\Sigma}(t) + \e(t).
\end{align*}
That is, $(\hat{A},\hat{K},\hat{C},I,\e)$ is a realization of $\y$ in forward innovation form. Since $(\hat{A},\hat{K})$ is controllable and $(\hat{C},\hat{A})$ is observable, $(\hat{A},\hat{K},\hat{C},I,\e)$
is a minimal realization of $\y$ in forward innovation form.
Hence, it is isomorphic to $(A_0,K_0,C_0,I,\e)$, that is,
there exists a nonsingular matrix $T$, such that $\hat{A}=TA_0T^{-1}$,
$\hat{K}=TK_0$ and $\hat{C}=C_0T^{-1}$, i.e., 
$\Sigma$ and $\Sigma(\theta_0)$ are isomorphic. 
Since $\Sigma=\Sigma(\theta)$ for some $\theta \in \Theta$, by Assumption \ref{pac:lti:assum}, $\theta=\theta_0$. 
\end{proof}
\begin{proof}[Proof of Lemma \ref{Psi:convergence}]
Under Assumption \ref{ass1},
by \cite[Theorem 2.3]{LjungBook} the alternative empirical loss converges to the generalization error, i.e.
$\lim_{N \rightarrow \infty} V_{\y,N}(f)=\mathcal{L}_{\y}(f)$
with probability $1$. 
Recall that by Lemma \ref{lem:rem:emploss}
$\lim_{N \rightarrow \infty} V_{\y,N}(f)-\hat{\mathcal{L}}_{\y,N}(f)=0$ in the mean sense,
hence there exists a subsequence $N_k$, $\lim_{k \rightarrow \infty} N_k=\infty$ such that
$\lim_{k \rightarrow \infty} V_{\y,N_k}(f)-\hat{\mathcal{L}}_{\y,N_k}(f)=0$ with probability $1$.
Hence, $\lim_{k \rightarrow \infty} \hat{\mathcal{L}}_{\y,N_k}(f)=\mathcal{L}_{\y}(f)$ with probability $1$.  Using dominated convergence theorem the
statement of the lemma follows.
\end{proof}
\begin{proof}[Proof of Lemma \ref{upperBound:lem1}]
Consider the matrix 
$Q_{\z,N}=\bE \left [\z_{1:N}(f_\theta)\z_{1:N}^T(f_\theta) \right ]$.
It is easy to see that $Q_{\z,N}$ is the left-upper block of the matrix
$Q_{\z,N+1}$. In particular, if $v \ne 0$ is an eigenvector of
$Q_{\z,N}$ which corresponds to eigenvalue $\rho_N(\theta)$, then
$\hat{v}=(v^T,0^T)$, $\hat{v}^T Q_{\z,N+1} \hat{v}= v^TQ_Nv = \rho_{N}(\theta) v^Tv > \rho_{N+1}(\theta) v^Tv$, where
we used the fact $\rho_{N+1}(\theta) w^Tw \le w^T Q_{N+1} w$,
due to the well-known property of smallest eigenvalue of a 
symmetric semi-definite matrix. Since $v \ne 0$, it follows
that $v^Tv > 0$, and hence 
$\rho_{N}(\theta) > \rho_{N+1}(\theta)$.
Since $\rho_N(\theta)$ is decreasing, it follows that
$\lim_{N \rightarrow \infty} \rho_N(\theta)=\rho^{*}(\theta)=\inf_N \rho_{N}(\theta)$.
From \cite[Section 2.2, eq. (2.1)]{CainesBook} it follows
$Q_{\z,N}=A_N Q_{\y,N} A_N^T$, where
$(Q_{\y,N})_{i,j}=E[\y(i)\y^T(j)]$ and 
$A_N=\begin{bmatrix} I & 0 & 0 & \cdots & 0 \\ -M_1 & I & 0 \cdots & 0 \\ \vdots &  \vdots & \vdots & \vdots & \vdots \\  -M_{N-1} & M_{N-2} & -M_{N-3} & \cdots & I \end{bmatrix}$,
where $M_i=C(A-KC)^{i-1}K$, $\theta=(A,K,C)$.
In particular, for any $v \ne 0$,
$v^TQ_{\z,N}v \ge \mu_{N} \|A_Nv\|^2 \ge \mu_N \|A_N^{-1}\|_2 v^Tv$,
where $\mu_N$ is the minimal eigenvalue of $Q_{\y,N}$.
Note that $A_N$ is trivially invertable, and
$A^{-1}_N=\begin{bmatrix} I & 0 & 0 & \cdots & 0 \\ \tilde{M}_1 & I & 0 \cdots & 0 \\ \vdots &  \vdots & \vdots & \vdots & \vdots \\  \tilde{M}_{N-1} & \tilde{M}_{N-2} & \tilde{M}_{N-3} & \cdots & I \end{bmatrix}$, where $\tilde{M}_i=CA^{i-1}K$.
It then follows that $\|A_N^{-1}\|_2 > 1$ and hence
$v^T Q_{\z,N}v > \mu_N v^Tv$, and therefore
$\rho_N(\theta) > \mu_N$. 
Finally, since $\y$ is coercive, by \cite[proof of Theorem 7.3]{Katayama:05}, it follows that
there exists $\mu_{*} > 0$ such that
the spectral density $\Phi_{\y}$ of $\y$
satisfies  $\Phi_{y}(e^{i\theta}) > \mu_* I$ for all $\theta \in [-\pi,\pi]$, and $\mu_N > \mu_* > 0$.
Finally,
notice that 
$v(\theta)=\bE[\|\y(t)-\hat{y}_{f{\theta}}(t)\|_2^2]=\mathrm{trace}\bE[(\y(t)-\hat{y}_{f{\theta}}(t))(\y(t) - \hat{y}_f{{\theta}}(t))^T]$, and
$\y(t)-\hat{\y}_{f_{\theta}}(t)$ is the result of
applying the linear system
$(A-KC,K,-C,I)$ to $\theta$, where $\theta=(A,K,C)$.
In particular, if $H(z)=I-C(zI-A-KC)^{-1}K$ is
the transfer function of that system, it follows 
by \cite[Theorem 1.2]{CainesBook} that
$H(z)\Phi_{\y}(z)H^{*}(z)$ is the spectral density of
$\y(t)-\hat{\y}_{f_{\theta}}(t)$, and 
$\bE[(\y(t)-\hat{y}_{f_{\theta}}(t))(\y(t) -\hat{y}_{f_{\theta}}(t))^T]=\frac{1}{2\pi} \int_{-\pi}^{\pi} H(e^{i\theta})\Phi_{\y}(e^{i\theta})H^{*}(e^{i\theta})d\theta$,
hence 
$v(\theta)=\frac{1}{2\pi} \int_{-\pi}^{\pi} \mathrm{trace} H(e^{i\theta})\Phi_{\y}(e^{i\theta})H^{*}(e^{i\theta})d\theta = 
\ge \mu_{*} \frac{1}{2 \pi} \int_{-\pi}^{\pi} \mathrm{trace} H(e^{i\theta})H^{*}(e^{i\theta})d\theta \ge p\mu_*$.
Here we used the fact that by Parseval's equality
$\frac{1}{2 \pi} \int_{-\pi}^{\pi} \mathrm{trace} H(e^{i\theta})H^{*}(e^{i\theta})d\theta=\|I_p\|^2_F+\sum_{k=0}^{\infty} \|C(A-KC)^{k}K\|^2_F > p \mu$, as
$H(e^{i\theta})=I_{p}+\sum_{k=0}^{\infty} C(A-KC)^kK e^{i k\theta}$ and $\|.\|_F$ denotes the Frobenius norm. 
\end{proof}
\begin{proof}[Proof of Lemma \ref{upperBound:lem2}]
  For every $\theta \in \Theta$, $(1+\frac{\lambda \rho_N(\theta)}{N/2})^{\frac{Np}{2}} \ge (1+\frac{\lambda \mu_*}{N/2})^{\frac{Np}{2}})$, and 
   $(1+\frac{\lambda \rho_N(\theta)}{N/2})^{\frac{Np}{2}} \le (1+\frac{\lambda \rho_0}{N/2})^{\frac{Np}{2}}$
    where we used the fact that 
    $\rho_0=\rho_1(\theta)$ and that by Lemma \ref{upperBound:lem1} $\rho_{N}(\theta)$ is decreasing.
  Hence, 
 $\frac{\exp{(\lambda v(\theta))}}{\left (1+\frac{\lambda\rho_N(\theta)}{\frac{N}{2}} \right)^{\frac{Np}{2}}}  \le \frac{\exp{(\lambda v(\theta))}}{\left (1+\frac{\lambda\mu_*}{\frac{N}{2}} \right)^{\frac{pN}{2}}}$, and
 $\frac{\exp{(\lambda v(\theta))}}{\left (1+\frac{\lambda\rho_N(\theta)}{\frac{N}{2}} \right)^{\frac{Np}{2}}}  \ge \frac{\exp{(\lambda v(\theta))}}{\left (1+\frac{\lambda\rho_0}{\frac{N}{2}} \right)^{\frac{pN}{2}}}$
  Notice that $\lim_{N \rightarrow \infty} (1+\frac{\lambda\mu_*}{\frac{N}{2}})^{\frac{pN}{2}}=\lim_{N \rightarrow \infty} (1+\frac{\lambda p\mu_*}{\frac{Np}{2}})^{\frac{pN}{2}}
  e^{\lambda p\mu_*}$
  and $\lim_{N \rightarrow \infty} (1+\frac{\lambda\rho_0}{\frac{N}{2}})^{\frac{pN}{2}}=\lim_{N \rightarrow \infty} (1+\frac{\lambda p\rho_0 }{\frac{Np}{2}})^{\frac{pN}{2}}
  e^{\lambda p\rho_0}$
  and $\frac{\exp{(\lambda v(\theta))}}{\left (1+\frac{\lambda p x}{\frac{pN}{2}} \right)^{\frac{pN}{2}}} \le \exp{(\lambda v(\theta))}$ for $x \in \{\rho_0, \mu_*, \rho_{N}(\theta) \mid N \ge 1\}$. The statement of the lemma follows now using
  dominated convergence theorem. 
\end{proof}
\begin{proof}[Proof of Lemma \ref{upperBound:lem3}]
The first equation follows by using the inequality
\[ \left(1 + \frac{x}{y} \right)^y > \exp\left(\frac{xy}{x+y} \right) \text{ for } x, y > 0
\]
with $x=p\lambda\rho_{N}(\theta)$ and $y=\frac{pN}{2}$,
the second one by noticing that $\frac{xy}{x+y}$ is monotonically increasing in $y$, hence,
$\frac{p\mu_* pN*0.5}{\lambda p\mu_* + 0.5pN} > \frac{p\mu_* 0.5}{\lambda p\mu_* + 0.5}$.
\end{proof}
\end{document}